\documentclass[letterpaper,10pt,conference]{ieeeconf} % Specifies the document style.
\pdfoutput=1
\IEEEoverridecommandlockouts
\makeatletter

\let\proof\@undefined
\let\endproof\@undefined
\makeatother

\usepackage{url}

\usepackage{amsmath,amssymb,amsthm}
\usepackage{graphicx,url}
\usepackage[font=footnotesize]{caption}
\usepackage{subcaption}
\usepackage{color}
\usepackage{xspace}
\usepackage{floatrow}
\usepackage{multirow}
\usepackage{algorithm}
\usepackage[noend]{algpseudocode}
\floatname{algorithm}{\normalsize Algorithm}
\usepackage{fixltx2e}

\DeclareFloatStyle{ruledtop}{precode=thickrule,midcode=rule,postcode=lowrule,capposition=TOP,heightadjust=object}
\newtheorem{theorem}{Theorem}[section]
\newtheorem{proposition}[theorem]{Proposition}

\newtheorem{lemma}[theorem]{Lemma}

\theoremstyle{definition}
\newtheorem{definition}[theorem]{Definition}
\newtheorem{problem}[theorem]{Problem}

\theoremstyle{remark}
\newtheorem{remark}[theorem]{Remark}

\newcommand{\W}{\mathcal{W}}
\newcommand{\Ropt}{R_{\texttt{OPT}}}
\urldef{\smith}\url{stephen.smith@uwaterloo.ca}
\urldef{\ahmad}\url{abasghar@uwaterloo.ca}

\DeclareMathOperator*{\argmin}{arg\,min}

\urldef{\smith}\url{stephen.smith@uwaterloo.ca}
\urldef{\ahmad}\url{abasghar@uwaterloo.ca}
\urldef{\shreyas}\url{sundara2@purdue.edu}
\title{\large \bf Multi-Robot Routing for Persistent Monitoring with Latency Constraints}

\author{Ahmad Bilal Asghar \qquad Stephen L. Smith \qquad Shreyas Sundaram\thanks{This research is partially
    supported by the Natural Sciences and Engineering Research Council
    of Canada (NSERC). }
  \thanks{Ahmad Bilal Asghar and Stephen L. Smith are with the Department of Electrical and
   Computer Engineering, University of Waterloo, Waterloo, ON N2L 3G1
   Canada.  Shreyas Sundaram is with the School of Electrical and
         Computer Engineering, Purdue University, West Lafayette, IN 47907 USA. }}    % Declares the author's name.

\date{}
\begin{document}
\maketitle

\begin{abstract}
In this paper we study a multi-robot path planning problem for persistent monitoring of an environment. We represent the areas to be monitored as the vertices of a weighted graph. For each vertex, there is a constraint on the maximum time spent by the robots between visits to that vertex, called the \emph{latency}, and the objective is to find the minimum number of robots that can satisfy these latency constraints. The decision version of this problem is known to be PSPACE-complete. We present a $O(\log \rho)$ approximation algorithm for the problem where $\rho$ is the ratio of the maximum and the minimum latency constraints. We also present an orienteering based heuristic to solve the problem and show through simulations that in most of the cases the heuristic algorithm gives better solutions than the approximation algorithm. We evaluate our algorithms on large problem instances in a patrolling scenario and in a persistent scene reconstruction application. We also compare the algorithms with an existing solver on benchmark instances. % and show that our heuristic algorithm runs $50$ to $1500$ times faster while returning equivalent solutions on $98\%$ of benchmark instances.
%\keywords{Motion and Path Planning, Approximation Algorithm}
\end{abstract}
\section{Introduction}
With the rapid development in field robotics, teams of robots can now perform long term monitoring tasks. Examples of such tasks include infrastructure inspection~\cite{cabrita2010infrastructure} to detect presence of anomalies or failures; patrolling for  surveillance~\cite{basilico2012patrolling,asghar2016stochastic} to detect threats in the environment; 3D reconstruction of scenes~\cite{roberts2017submodular,bircher2015structural} in changing environments; informative path planning~\cite{cao2013multi} for observing dynamic properties of an area; and forest fire monitoring~\cite{merino2012unmanned}. In such persistent monitoring scenarios, locations in an environment need to be visited repeatedly by a team of robots. Since the duration of the events, or the rate of change of the properties to be monitored, can be different for different locations, each location will have a different latency constraint, which specifies the maximum time allowed between consecutive visits to that location. In this paper we study the problem of finding a set of paths that continually visit a set of locations while collectively satisfying the latency constraints on each location.

We represent the locations of interest in the environment as the vertices of a graph. The edge weights give the travel time between the vertices and each vertex has a latency constraint that defines the maximum allowed time between two consecutive visits to that vertex. The problem is to find walks on the graph for the minimum number of robots that can satisfy the latency constraints.

\emph{Related Work:} Persistent monitoring problems have been extensively studied in the literature~\cite{hokayem2007persistent,nigam2008persistent,nigam2012control,elmaliach2009multi}. In~\cite{palacios2016multi}, persistent coverage using multiple robots in a continuous environment is considered. The problem of determining a visit sequence for a set of locations along with the time spent at each location to gather information is considered in~\cite{yu2018optimal,yu2017optimal}. For the problem with latency constraints, the authors in~\cite{las2013persistent} use incomplete greedy heuristics to find if a single robot can satisfy the constraints on all vertices of a graph. They show that if a solution exists, then a periodic solution also exists. In this paper, we consider the multi-robot problem and our objective is to minimize the number of robots that can satisfy the latency constraints on the given graph. This problem has been considered in~\cite{drucker2016cyclic,ho2015cyclic}, where it is called \emph{Cyclic Routing of Unmanned Aerial Vehicles}. The decision version of the problem for a single robot is shown to be PSPACE-complete in~\cite{ho2015cyclic}. The authors also show that the length of even one period of a feasible walk  can be exponential in the size of the problem instance. In~\cite{drucker2016cyclic}, the authors propose a solver based on
% give a lower bound on the number of robots required and use a
Satisfiability Modulo Theories (SMT). To apply an SMT solver, they impose an upper bound on the length of the period of the solution.  Since an upper bound is not known \emph{a priori}, the solver will not return the optimal solution if the true optimal period exceeds the bound.  The authors generate a library of test instances, but since their algorithm scales exponentially with the problem size, they solve instances up to only 7 vertices. We compare our algorithms with their solver and show that our algorithms run over $500$ times faster on average and return solutions with the same number of robots on $98\%$ of the benchmark instances provided by~\cite{drucker2016cyclic}.

A related single robot problem is studied in~\cite{Alamdari2014} where each vertex has a weight associated with it and the objective is to minimize the maximum weighted latency (time between consecutive visits) for an infinite walk. The authors provide an approximation algorithm for this problem. The authors in~\cite{basilico2012patrolling} study the single robot problem for a security application where the length of attack on each vertex of the graph is given.  To intercept all possible attacks, they design an algorithm to repeatedly patrol all vertices with the maximum revisit time to each vertex less than its length of attack.

% propose an algorithm based on the premise that the problem is NP-Complete that is shown to be wrong in~\cite{ho2015cyclic}.

Timed automata have been used to model general multi-robot path planning problems~\cite{quottrup2004multi} as the clock states can capture the concurrent time dependent motion. In~\cite{ulusoy2013optimality}, temporal logic constraints are used to specify high level mission objectives to be achieved by a set of robots. The routing problem with latency constraints can also be modeled as a timed automaton since multiple robots may require synchronization to satisfy the latency constraints. A timed automaton based solution to the problem is presented in~\cite{Drucker2014thesis}, however it is shown to perform more poorly than the SMT-based approach in~\cite{drucker2016cyclic}, which we use as a comparison for our proposed algorithms.

Several vehicle routing problems are closely related to persistent monitoring with latency constraints. In the \emph{vehicle routing problem with time windows}~\cite{braysy2005vehicle}, customers have to be served within their time windows by several vehicles with limited capacity. The goal is to minimize the number of vehicles required. Since the problem does not require repeated visits, the length of the resulting tour is polynomially bounded, and thus the problem is in NP. In the \emph{deadline-TSP}~\cite{tsitsiklis1992special}, the vertices have deadlines for first visits. In the \emph{period vehicle routing problem}~\cite{christofides1984period}, the problem is to design routes for each day of a given period where each customer may require a number of visits (in a given number of allowable combinations) during this period. The main difference between these problems and the cyclic routing problem with latency constraints is that the latency constraints need to be satisfied indefinitely which makes it harder than these problems.

\emph{Contributions:} We present a $O(\log \rho)$ approximation algorithm for the problem where $\rho$ is the ratio of the maximum and the minimum latency constraints (Section~\ref{sec:approx}). We also provide an algorithm for the problem of minimizing the maximum weighted latency with multiple robots and show that an approximation algorithm for this problem yields a bi-criterion approximation for our problem. We present an orienteering based heuristic algorithm to solve the problem and prove its completeness (Section~\ref{sec:greedy}). We show through simulations that the heuristic algorithm gives better solutions than the approximation algorithm. We evaluate the performance of the algorithms on large problem instances in a patrolling scenario and in an image collection application for 3D scene reconstruction. We also compare our algorithms against an existing solver on benchmark instances (Section~\ref{sec:sim}).

\section{Background and Notation}
\label{sec:background}
Given a directed graph $G=(V,E)$ with edge lengths $l(e)$ for each $e\in E$, a \emph{simple walk} in graph $G$ is defined as a sequence of vertices $(v_1,v_2,\ldots,v_k)$ such that $(v_i,v_{i+1})\in E$ for each $1\leq i < k$. An \emph{infinite walk} is a sequence of vertices $(v_1,v_2,\ldots)$ such that $(v_i,v_{i+1})\in E$ for each $i\in \mathbb{N}$. Given walks $W_1$ and $W_2$, $[W_1 ,W_2]$ represents the concatenation of the walks. Given a finite walk $W$, an infinite periodic walk is constructed by concatenating infinite copies of W together, and is denoted by $\Delta (W)$. A \emph{cycle} is a simple walk that starts and ends at the same vertex with no other vertex appearing more than once.

In general, a walk can stay for some time at a vertex before traversing the edge towards the next vertex. Therefore we define a \emph{timed walk} $W$ in graph $G$ as a sequence $(o_1,o_2,\ldots,o_k)$, where $o_i=(v_i,t_i)$ is an ordered pair that represents the holding time $t_i$ that the walk $W$ spends at vertex $v_i$, such that $(v_i,v_{i+1})\in E$ for each $1\leq i < k$. The definitions of infinite walk and periodic walk can be extended to infinite timed walk and periodic timed walk. A walk with ordered pairs of the form $(v_i,0)$ becomes a simple walk. The vertices traversed by walk $W$ are given by $V(W)$ and the length of walk $W=((v_1,t_1),(v_2,t_2),\ldots,(v_k,t_k))$ is given by $l(W) = \sum_{i=1}^{k-1}{l(v_i,v_{i+1})} + \sum_{i=1}^{k}{t_k}$.
Since we are considering a multi-robot problem, synchronization between the walks is important. Given a set of walks $\W = \{W_1,W_2,\ldots, W_k\}$ on graph $G$, we assume that at time $0$, each robot $i$ is at the first vertex $v^i_1$ of its walk $W_i$, and will spend the holding time $t^i_1$ at that vertex before moving to $v^i_2$.

Given a graph $G$ and length $\lambda$, the\textsc{ Minimum Cycle Cover Problem} (MCCP) is to find minimum number of cycles that cover the whole graph such that the length of the longest cycle is at most $\lambda$. This problem is NP-hard and a $14/3$-approximation algorithm for MCCP is given in~\cite{yu2016improved}.

Given a graph $G=(V,E)$ with vertex weights $\psi_i$ for $i\in V$, and length $\lambda$, the \textsc{Orienteering} problem is to find a path from vertex $s$ to $t$ of length at most $\lambda$ such that the sum of the weights on the vertices in the path is maximized. This problem is also NP-hard and a $(2+\epsilon)$-approximation is given in~\cite{chekuri2012improved}.

\section{Problem Statement}
\label{sec:problem}
Let $G=(V,E)$ be a directed weighted graph with edge lengths $l(e)$ for each $e\in E$. The edge lengths are metric and represent the time taken by the robot to travel between the vertices. The latency constraint for each vertex $v$ is denoted by $r(v)$ and represents the maximum time allowed between consecutive visits to $v$. The time taken by the robots to inspect a vertex can be added to the length of the incoming edges of that vertex to get an equivalent metric graph with zero inspection times and modified latency constraints~\cite{drucker2016cyclic}. Hence, we assume that the time required by the robots to inspect a vertex is zero.  We formally define the problem statement after the following definition.

\begin{definition}[Latency]
Given a set of infinite walks $\W = \{W_1,W_2,\ldots, W_k\}$ on a graph $G$, let $a^v_i$ represent the $i^{th}$ arrival time for the walks to vertex $v$. Similarly, let $d^v_i$ represent the $i^{th}$ departure time from vertex $v$. Then the latency $L(\W,v)$ of vertex $v$ on walks $\W$ is defined as the maximum time spent away from vertex $v$ by the walks, i.e., $L(\W,v) = \sup_{i}{(a^v_{i+1}-d^v_i)}$.
\end{definition}

%\textbf{Minimizing Maximum Weighted Latency:} Given the weight $\phi(v)$ for all vertices $v\in V$, the weighted latency of $v$ is defined as $C(\W,v)=\phi(v)L(\W,v)$. Given the number of robots $R$, the problem of minimizing maximum weighted latency is to find a set of $R$ infinite walks $\W=\{W_1,W_2,\ldots,W_R\}$ such that the cost $\max_v C(\W,v)$ is minimized. Without loss of generality, $\phi(v)$ is assumed to be normalized such that $\max_v\phi(v)=1$.

\begin{problem}[Minimizing Robots with Latency Constraints]
\label{pbm:min_robs} Given the latency constraints $r(v), \forall v\in V$, the optimization problem is to find a set of walks $\W$ with minimum cardinality such that $L(\W,v) \leq r(v), \forall v \in V$.
\end{problem}
The decision version of the problem is to determine whether there exists a set of $R$ walks $\W=\{W_1,W_2,\ldots,W_R\}$ such that $L(\W,v) \leq r(v)$ for all $v \in V$. Note that a general solution to Problem~\ref{pbm:min_robs} will be a set of timed walks with possibly non-zero holding times.

In this problem definition, the graph and its edge lengths capture the robot motion in the environment.  This graph can be generated from a probabilistic roadmap, or any other environment decomposition method.  The latency constraints provide the maximum allowable time between visits to a vertex. For example, in dynamic scene reconstruction, each vertex corresponds to a viewpoint~\cite{roberts2017submodular}. The latency constraints may encode the maximum staleness of information that can be tolerated for the voxels captured at that viewpoint.
%Although we assume that the inspection time of vertices is zero, note that the holding times of the walks in a solution might still be non zero to achieve synchronization among the walks.

%\subsection{Complexity of the Problem}
%The decision version of the problem has been shown to be PSPACE-complete in~\cite{ho2015cyclic} even for a single robot. In~\cite{las2013persistent}, the authors show that if a solution exists for the decision version of the problem, then a periodic solution also exists for that problem. This means that instead of searching for infinite walks as the solutions to the problem, we can search for finite walks $W$ such that $\Delta(W)$ is a solution to the problem. However the length of even one period of the feasible walk can be exponential in the size of the problem instance~\cite{ho2015cyclic}.

\subsection{Multiple Robots on the Same Walk}\label{sec:factor_R}
In multi-robot problems that involve finding cycles or tours in a graph, equally placing $n$ robots on a tour reduces the cost of that tour by a factor of $n$~\cite{chevaleyre2004theoretical}. That is not true for Problem~\ref{pbm:min_robs}: if a periodic walk $W$ gives latency $L(W,v)$ on vertex $v$, equally spacing more robots on one period of that walk does not necessarily reduce the latency for that vertex. Figure~\ref{fig:example1} gives an example of such a walk. The latency of vertices $a,b$ and $c$ on the walk $(a,b,a,c)$ are $2,4$ and $4$ respectively. The length of one period of the walk is $4$. If we place another robot following the first robot with a lag of $2$ units, the latency of vertex $a$ remains the same. If we place the second robot at a lag of $1$ unit, the latency will reduce to $1$ for vertex $a$ and $3$ for vertices $b$ and $c$.
Hence we need more sophisticated algorithms than finding a walk for a single robot and adding more robots on that walk until the constraints are satisfied, unless that walk is a cycle.
\begin{figure}[t]
\centering
%\vspace*{1mm}
\includegraphics[width=.4\linewidth]{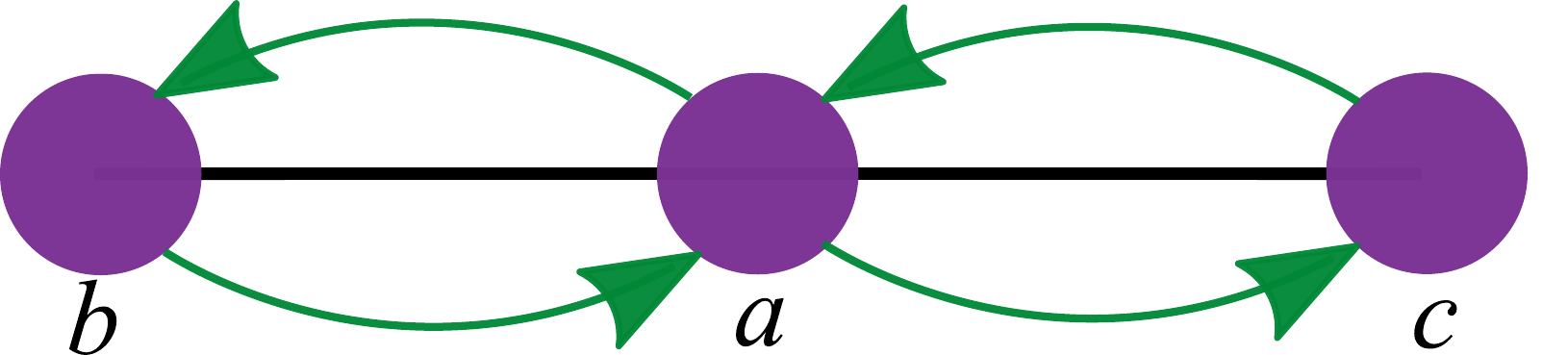}
\caption{A graph with three vertices and the walk $(a,b,a,c)$. The length of shown edges is one. Equally placing two robots on this walk does not halve the latencies.}
\label{fig:example1}
\end{figure}

\section{Approximation Algorithm}
\label{sec:approx}
%\subsection{Bi-criterion Approximation}
\iffalse
Idea from extreme case: What if we find a TSP of the whole graph and then equally place the robots to decrease latency of each vertex until the minimum latency is satisfied.

What if I just divide the levels and equally place on a TSP? Then the question is that how many robots will I need to place on that TSP to satisfy all constraints. depends on the size of TSP.

If TSP length is more that the maximum required latency in that partition, then one robot could not have been optimal and we would need MCCP result.
\fi

Since Problem~\ref{pbm:min_robs} is PSPACE-complete, we resort to finding approximate and heuristic solutions to the problem. In this section, we present an approximation algorithm for the problem.
\subsection{\text{$O(\log\rho)$} Approximation}
We first mention a simple approach to the problem and then improve it incrementally to get the approximation algorithm. One naive solution is to find a TSP tour of the graph and equally place robots on that tour to satisfy all the latency constraints. However, a single vertex with a very small $r(v)$ can result in a solution with the number of robots proportional to $1/r(v)$. To solve this issue, we can partition the vertices of the graph such that the latencies in one partition are close to each other, and then place robots on the TSP tour of each partition. If more than one robot is required for a partition $V'$, then another approach is to solve the MCCP for that partition. The benefit of using the MCCP over placing multiple robots on a TSP is that if all the vertices in $V'$ had the same latency requirement, then we have a guarantee on the number of cycles required for that partition. However, a general solution to the problem might not consist of simple cycles. Lemma~\ref{lem:yann} relates solutions consisting of cycles to general solutions and shows that a solution consisting of cycles will have latencies no more than twice that of any general solution with same number of robots. Therefore, if we solve the MCCP on a partition with its latency constraints multiplied by two, we have a guarantee on the number of cycles. We can then use the naive idea from TSP based solutions and place multiple robots on each cycle to satisfy the latency constraints.

The approximation algorithm is given in Algorithm~\ref{approx_latency}. The first four lines of the algorithm partition the vertices according to their latency constraints. For each of those partitions, the function MCCP$(V, \lambda)$ called in line~\ref{algln:MCCP} uses an approximation algorithm for the minimum cycle cover problem from~\cite{yu2016improved}. Then, the appropriate number of robots are placed on each cycle returned by the MCCP function to satisfy the latency constraints. We will need the following definition to establish the approximation ratio of Algorithm~\ref{approx_latency}. A similar relaxation technique was used in~\cite{Alamdari2014}.
%The function MCCP$(V, \lambda)$ called in line~\ref{algln:MCCP} of the algorithm uses one of the approximation algorithms for minimum cycle cover problem from~\cite{yu2016improved}. We will need the following definition to establish the approximation ratio of Algorithm~\ref{approx_latency}. A similar relaxation technique was used in~\cite{Alamdari2014}.
\begin{algorithm}[t]
	\openup -.1em
  \caption{\textsc{ApproximationAlgorithm}}
	\label{approx_latency}
	\begin{algorithmic}[1]
	\Statex Input: Graph $G=(V,E)$, latency constraints $r(v), \forall v$
	\Statex Output: A set of walks $\W$, such that $L(\W,v)\leq r(v)$
  \vspace{0.2em}
  \hrule
  \vspace{0.2em}
	\State Let $r_{\texttt{max}}=\max_v r(v)$, $r_{\texttt{min}}=\min_v r(v)$, $\rho = \frac{r_{\texttt{max}}}{ r_{\texttt{min}}}$
	\If {$r_{\texttt{max}}/ r_{\texttt{min}}$ is an exact power of $2$} $\rho = \frac{r_{\texttt{max}}}{ r_{\texttt{min}}}+1$ \EndIf
	\State $\W = \{\}$
	%\State Find integer $x$ such that $x\lceil \log_x \rho\rceil$ is minimized
	\State Let $V_{i}$ be the set of vertices $v$ such that  $ r_{\texttt{min}}2^{i-1}\leq r(v) < r_{\texttt{min}}2^{i} $  for $1\leq i \leq \lceil \log_2 \rho \rceil $\label{algln:levels}
	\For {$i=1$ to $\lceil \log_2 \rho \rceil$}
	\State  $\mathcal{C} =$ MCCP$(V_{i}, r_{\texttt{min}}2^{i+1})$
	\label{algln:MCCP}
	\For {$C$ in $\mathcal{C}$}
	\State Equally place $\lceil l(C)/\min_{v\in V(C)}{r(v)} \rceil$ robots on \hspace*{9.5mm} cycle $C$ to get walks $\W'$\label{algln:final}
	\State $\W=\{\W,\W'\}$
	\EndFor
	\EndFor
	\end{algorithmic}
\end{algorithm}

\begin{definition}
Let $r_{\texttt{min}}=\min_v r(v)$. The latency constraints of the problem are said to be relaxed if for any vertex $v$, its latency constraint is updated from $r(v)$ to $\bar{r}(v) = r_{\texttt{min}}2^{x}$ such that $x$ is the smallest integer for which $r(v) < r_{\texttt{min}}2^{x}$.
\end{definition}
We will also need the following lemma that follows from Lemma 2 in~\cite{chevaleyre2004theoretical} and we omit the proof for brevity.
\begin{lemma}
\label{lem:yann}
For any set of walks $\mathcal{W}$ on an undirected metric graph with vertices $V$, there exists a set of walks $\mathcal{W}'$ on $V$ with $|\mathcal{W}|=|\mathcal{W}'|$, such that each walk $W_i \in \mathcal{W}'$ is a cycle of vertices $V_j\subseteq V$, and the sets $V_j$ partition $V$, and $\max_v L(\mathcal{W}',v) \leq 2 \max_vL(\mathcal{W},v)$.
\end{lemma}
%\margin{see if it can be improved using the approximation proof in MCCP paper.}
The following proposition gives the approximation factor of Algorithm~\ref{approx_latency}.
\begin{proposition}
\label{prop:approx}
Given an undirected metric graph $G=(V,E)$ with latency constraints $r(v)$ for $v\in V$, Algorithm~\ref{approx_latency} constructs $R$ walks $\W=\{W_1,W_2,\ldots, W_R\}$ such that $L(\W,v)\leq r(v)$ for all $v\in V$ and $R \leq 4\alpha\lceil\log(\rho) \rceil\Ropt$, where $\Ropt$ is the minimum number of robots required to satisfy the latency constraints and $\alpha$ is the approximation factor of MCCP.
\end{proposition}

\begin{proof}
Given that $\Ropt$ robots will satisfy the latency constraints $r(v)$, they will also satisfy the relaxed constraints $\bar{r}(v)$ since $\bar{r}(v)>r(v)$. Therefore, there exists a set of at most $\Ropt$ walks $\W^*$ such that for $v\in V_i$, $L(\W^*,v)\leq r_{\texttt{min}}2^{i}$.

Using Lemma~\ref{lem:yann}, given the set $\W^*$, $\Ropt$ cycles can be constructed in $V_i$ such that the latency of each vertex in $V_i$ is at most $r_{\texttt{min}}2^{i+1}$. Hence, running an $\alpha$ approximation algorithm for Minimum Cycle Cover Problem (MCCP) on the subgraph with vertices $V_i$ and with maximum cycle length $r_{\texttt{min}}2^{i+1}$ will not return more than $\alpha \Ropt$ cycles.

Since MCCP returns cycles, equally placing $k$ robots on each cycle will reduce the latency of each vertex on that cycle by a factor of $k$. As $r(v) \geq  r_{\texttt{min}}2^{i+1}/4$ for each $v\in V_i$, we will need to place at most $4$ robots on each cycle.

Finally, since there are at most $\lceil\log \rho\rceil$ partitions $V_i$, the algorithm will return   $R \leq 4\alpha \lceil\log(\rho)\rceil \Ropt$ walks.
\end{proof}

\textbf{Runtime:} Since we run the approximation algorithm for MCCP on partitions of the graph, the runtime of Algorithm~\ref{approx_latency} is the same as that of the approximation algorithm of MCCP. That is because the runtime of MCCP is superlinear, so if $\sum |V_i| = |V|$, then $\sum |V_i|^p \leq |V|^p $ for $p\geq 1$.
\begin{remark}[Heuristic Improvements]
Instead of finding cycles using MCCP for each partition $V_i$ in line~\ref{algln:MCCP} of the algorithm, we can also equally place robots on the Traveling Salesman Tour of $V_i$ to get a feasible solution. In practice, we use both of these methods and pick the solution that gives the lower number of robots for each $V_i$. This modification can return better solutions to the problem but does not improve the approximation guarantee established in Proposition~\ref{prop:approx}.
\end{remark}
\subsection{Relation to Min Max Weighted Latency}
The approximation algorithm and analysis presented above helps in formulating an algorithm for the multi-robot version of a related problem. In~\cite{Alamdari2014}, the authors study in detail the problem of minimizing the maximum weighted latency of a graph given a single robot. Here we define the multiple robot version of the problem.
\begin{problem}[Minimizing Maximum Weighted Latency:]
\label{pbm:weighted_latency} Given a graph $G=(V,E)$ with weights $\phi(v)$ for $v\in V$, and a set of walks $\mathcal{W}$, the weighted latency of $v$ is defined as $C(\W,v)=\phi(v)L(\W,v)$. Given the number of robots $R$, the problem of minimizing maximum weighted latency is to find a set of $R$ infinite walks $\W=\{W_1,W_2,\ldots,W_R\}$ such that the cost $\max_v C(\W,v)$ is minimized.
\end{problem}
Without loss of generality, $\phi(v)$ is assumed to be normalized such that $\max_v\phi(v)=1$. In this section we relate this problem to Problem~\ref{pbm:min_robs}.

In~\cite{Alamdari2014}[Algorithm 2], an approximation algorithm for the single robot version of Problem~\ref{pbm:weighted_latency} is given. We will refer to that approximation algorithm as \textsc{MinMaxLatencyOneRobot}$(G_j)$, which returns a walk in graph $G_j$ such that the maximum weighted latency of that walk is not more than $(8\log\rho_j + 10)\texttt{OPT}_1^j$, where $\rho_j$ is the ratio of maximum to minimum vertex weights in $G_j$ and $\texttt{OPT}_1^j$ is the optimal maximum weighted latency for one robot in $G_j$. We use this approximation algorithm as a subroutine in Algorithm~\ref{multiple_latency} for minimizing the maximum weighted latency with multiple robots.

\begin{algorithm}
\openup -.1em
	\caption{\textsc{LatencyWalks}}
	\label{multiple_latency}
	\begin{algorithmic}[1]
	\Statex Input: Graph $G=(V,E)$, vertex weights $\phi(v), \forall v\in V$, and number of robots $R$.
	\Statex Output: A set of $R$ walks $\{W_1,\ldots,W_R\}$ in $G$.
  \vspace{0.2em}
  \hrule
  \vspace{0.2em}
	\State $\rho = \max_{i,j} \phi_i/\phi_j$
	\If {$\max_{i,j} \phi_i/\phi_j$ is an exact power of $2$} $\rho = \max_{i,j} \phi_i/\phi_j+1$ \EndIf
	\State Let $V_{i}$ be the set of vertices of weight $\frac{1}{2^i} < \phi(u) \leq \frac{1}{2^{i-1}} $  for $1\leq i \leq \lceil \log_2 \rho \rceil $
	\If {$R<\log\rho$}
	\For {$j=1$ to $R$}
	\State Let $G_j$ be a subgraph of $G$ with vertices $V_i$ for \hspace*{8.5mm} $\lceil\frac{j-1}{R}\log\rho\rceil +1 \leq  i\leq \lceil\frac{j}{R}\log\rho\rceil$\label{algln:Gj}
	\State $W_j = $ \textsc{MinMaxLatencyOneRobot}$(G_j)$
	\EndFor
	\EndIf
	\If {$R \geq \log\rho$}
	\State Equally space $\lfloor R/\lceil\log\rho\rceil \rfloor$ robots on TSP tour of $V_i$ \hspace*{5.5mm}for all $i$ to get  $\{W_1,\ldots,W_{ \lceil\log\rho\rceil \lfloor \frac{R}{\lceil\log\rho\rceil }\rfloor}\}$
	\For {$k=R-\lceil\log\rho\rceil \lfloor \frac{R}{\lceil\log\rho\rceil }\rfloor+1$ to $R$}
	\State Find subset $V_i$ that has the maximum cost with \hspace*{10.5mm}currently assigned robots
	\State Equally space all the robots on $V_i$ along with \hspace*{10.5mm}robot $k$ to get $W_k$
	\EndFor
	\EndIf
	\end{algorithmic}
\end{algorithm}

\begin{proposition}
Given an instance of Problem~\ref{pbm:weighted_latency}, Algorithm~\ref{multiple_latency} constructs $R$ walks such that the maximum weighted latency of the graph is not more than $(\frac{8\log\rho}{R}+10)\texttt{OPT}_1$ if $R\leq \log\rho$ and $3\texttt{OPT}_1/\lfloor R/\lceil\log\rho\rceil \rfloor$ otherwise, where $\rho=\max \frac{\phi (v_i)}{\phi (v_j)}\ $ and $\texttt{OPT}_1$ is the maximum weighted latency of the single optimal walk.
\end{proposition}
\begin{proof}
The maximum vertex weight in the subgraph $G_j$ constructed at line~\ref{algln:Gj} of the algorithm will be at most $1/(2^{\frac{j-1}{R}\log\rho})$, whereas the minimum vertex weight in $G_j$ will be at least $1/(2^{\frac{j}{R}\log\rho})$. Hence the ratio of the maximum to minimum vertex weights in $G_j$ will be at most $\rho_j = 2^{\frac{\log\rho}{R}}$. Therefore, the approximation algorithm for one robot will return a walk $W_j$ such that the maximum weighted latency of $W_j$ will not be more than $(8\log\rho_j + 10)\texttt{OPT}_1^j$. Moreover, $\texttt{OPT}_1^j \leq \texttt{OPT}_1$ and hence if $R<\log \rho$, the maximum weighted latency will be at most $(\frac{8\log\rho}{R} + 10)\texttt{OPT}_1$.

The TSP tour of $V_i$ is an optimal solution when all the vertex weights in $V_i$ are equal. Since the vertex weights within $V_i$ differ by a factor of $2$ at most, and the approximation factor of TSP tour in metric graphs is $3/2$, the maximum weighted latency of the TSP tour will be at most $3 \texttt{OPT}_1$. Equally placing $\lfloor R/\lceil\log\rho\rceil \rfloor$ will decrease the latency of all the vertices by a factor of $\lfloor R/\lceil\log\rho\rceil \rfloor$.
\end{proof}
Note that Algorithm~\ref{multiple_latency} bounds the cost of the solution by a function of the optimal cost of a single robot. This algorithm shows that $R$ robots can asymptotically decrease the weighted latency given by a single walk by a factor of $R$, which is not straightforward for this problem as discussed in Section~\ref{sec:factor_R}.

Now we show that if there is an approximation algorithm for Problem~\ref{pbm:weighted_latency}, it can be used to solve Problem~\ref{pbm:min_robs} using the optimal number of robots but with the latency constraints relaxed by a factor $\alpha$. This is referred to as a $(\alpha,1)$-bi-criterion algorithm~\cite{iyer2013submodular} for Problem~\ref{pbm:min_robs}.
\begin{proposition}\label{prop:bicriterion}
If there exists an $\alpha$-approximation algorithm for Problem~\ref{pbm:weighted_latency}, then there exists a $(\alpha,1)$-bi-criterion approximation algorithm for Problem~\ref{pbm:min_robs}.
\end{proposition}
We will need the following lemma relating the two problems to prove the proposition. Given an instance of the decision version of Problem~\ref{pbm:min_robs} with $R$ robots, let us define an instance of Problem~\ref{pbm:weighted_latency} by assigning $\phi(v)= \frac{r_{\texttt{min}}}{r(v)}, \forall v\in V$, where $r_{\texttt{min}}=\min_v r(v)$.
\begin{lemma}
\label{lem:walk}
 An instance of the decision version of Problem~\ref{pbm:min_robs} is feasible if and only if the optimal maximum weighted latency is at most $r_{\texttt{min}}$ for the corresponding instance of Problem~\ref{pbm:weighted_latency}.
\end{lemma}
\begin{proof}
If the optimal set of walks $\W$ has a cost more than $r_{\texttt{min}}$, then $L(\W,v) > r_{\texttt{min}}/\phi(v) =r(v)$ for some vertex $v$. Hence the latency constraint for that vertex is not satisfied and the set of walks $\W$ is not feasible.

If the optimal set of walks $\W$ has a cost at most $r_{\texttt{min}}$, then $ L(\W,v)\phi(v) \leq r_{\texttt{min}}$ for all $v$. Hence, $L(\W,v) \leq r_{\texttt{min}}/\phi(v) = r(v)$. So, the latency constraints are satisfied for all vertices and $\W$ is feasible.
\end{proof}
%The following lemma shows that if we are constrained by the number of robots $R$, then an $\alpha$-approximation algorithm for weighted latency problem can give $R$ walks such that each of the latency constrains is violated by at most a factor of $\alpha$.
%The following proposition shows that if there is an approximation algorithm for Problem~\ref{pbm:weighted_latency}, it can be used for a bi-criterion approximation scheme for Problem~\ref{pbm:min_robs}.

\begin{proof}[Proof of Proposition~\ref{prop:bicriterion}]
If a problem instance of Problem~\ref{pbm:min_robs} with $R$ robots is feasible, then by Lemma~\ref{lem:walk} the optimal set of walks has a cost at most $r_{\texttt{min}}$. The $\alpha$-approximation algorithm for the corresponding Problem~\ref{pbm:weighted_latency} will return a set of walks $\W$ with a cost no more than $\alpha r_{\texttt{min}}$. Hence, $L(\W,v)\leq \alpha r_{\texttt{min}}/\phi(v) =\alpha r(v)$, for all $v$.

Hence, we can use binary search to find the minimum number of robots for which the $\alpha$-approximation algorithm for the corresponding Problem~\ref{pbm:weighted_latency} results in a latency at most $\alpha r(v)$ for all $v$. This will be the minimum number of robots for which the problem is feasible.
\end{proof}

\section{Heuristic Algorithms}
\label{sec:greedy}
The approximation algorithm for Problem~\ref{pbm:min_robs} presented in Section~\ref{sec:approx} is guaranteed to provide a solution within a fixed factor of the optimal solution. In this section, we propose a heuristic algorithm based on the orienteering problem, which in practice provides high-quality solutions.
\subsection{Partitioned Solutions}
In general, walks in a solution of the problem may share some of the vertices. However, sharing the vertices by multiple robots requires coordination and communication among the robots. Such strategies may also require the robots to hold at certain vertices for some time before traversing the next edge, in order to maintain synchronization. This is not possible for vehicles that must maintain forward motion, such as fixed-wing aircraft. The following example illustrates that if vertices are shared by the robots, lack of coordination or perturbation in edge weights can lead to large errors in latencies.\\
\textbf{Example:} Consider the problem instance shown in Figure~\ref{fig:example2}. An optimal set of walks for this problem is given by $\{W_1,W_2,W_3\}$ where $W_1 = ((a,1),(b,1))$, $W_2=((b,0),(c,0))$ and $W_3 = ((c,0),(d,1),(c,1))$. Note that walk $W_1$ starts by staying on vertex $a$, while $W_2$ leaves vertex $b$ and $W_3$ leaves vertex $c$. Also note that any partitioned solution will need $4$ robots. Moreover, if the length of edge $\{b,c\}$ changes from $3$ to $3-\epsilon$, (e.g., if the robot's speed increases slightly) the latencies of vertices $b$ and $c$ will keep changing with time and will go up to $5$. Hence, a small deviation in robot speed can result in a large impact on the monitoring objective.
\begin{figure}[h]
\centering
\includegraphics[width=.6\linewidth]{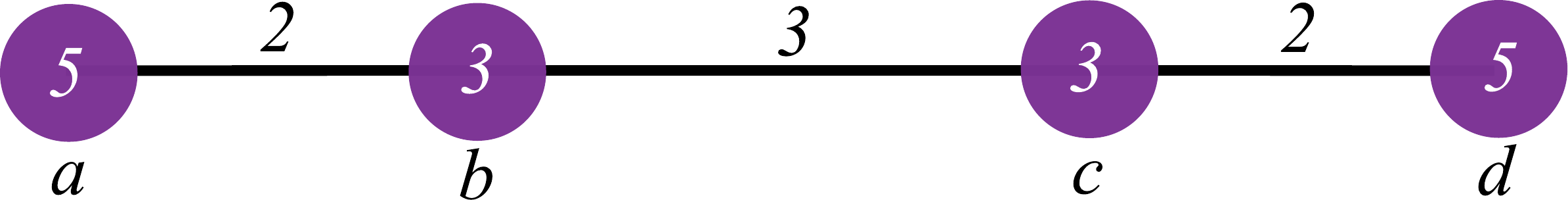}
\caption{A problem instance with an optimal set of walks that share vertices. The latency constraints for each vertex are written inside that vertex. The edge lengths are labeled with the edges. The optimal walks are $\{W_1,W_2,W_3\}$ where $W_1 = ((a,1),(b,1))$, $W_2=((b,0),(c,0))$ and $W_3 = ((c,0),(d,1),(c,1))$.}
\label{fig:example2}
\end{figure}

Since the above mentioned issues will not occur if the robots do not share the vertices of the graph, and the problem is PSPACE-complete even for a single robot, we focus on finding partitioned walks in this section. {The general greedy approach used in this section is to find a single walk that satisfies latency constraints on a subset of vertices $V'\subseteq V$. Note that we do not know $V'$ beforehand, but a feasible walk on a subset of vertices will determine $V'$. We then repeat this process of finding feasible walks on the remaining vertices of the graph until the whole graph is covered.

\subsection{Greedy Algorithm}
\label{sec:simple_greedy}
We now consider the problem of finding a single walk on the graph $G=(V,E)$ that satisfies the latency constraints on the vertices in $V'\subseteq V$. Given a robot walking on a graph, let $p(k)$ represent the vertex occupied by the robot after traversing $k$ edges (after $k$ steps) of the walk. Also, at step $k$, let the maximum time left until a vertex $i$ has to be visited by the robot for its latency to be satisfied be represented by $s_i(k)$. If that vertex is not visited by the robot within that time, we say that the vertex expired. Hence, the vector $s(k)=[s_1(k),\ldots,s_{|V'|}(k)]^T$ represents the time to expiry for each vertex. At the start of the walk, $s_i(0)=r(i)$, and $s_i(k)$ evolves according to the following equation:
\begin{equation}
\label{eq:slackupadate}
s_i(k) = \begin{cases}
r(i) \quad  &\quad \text{if } p(k)=i \\
s_i(k-1)- l(p(k-1),p(k)) & \quad \text{otherwise}. \\
\end{cases}
\end{equation}

We will use the notation $s_i$ without the step $k$ if it is clear that we are talking about the current time to expiry. An incomplete greedy heuristic for the decision version of the problem with $R=1$ is presented in~\cite{las2013persistent}. The heuristic is to pick the vertex with minimum value of $s_i(k)$ as the next vertex to be visited by the robot. This heuristic does not ensure that all the vertices on the walk will have their latency constraints satisfied since the distance to a vertex $i$ to be visited might get larger than $s_i(k)$. To overcome this, we propose a modification to the heuristic to apply it to our problem. Given a walk $W$ on graph $G$, the function \textsc{PeriodicFeasibility$(W,G)$} determines whether the periodic walk $\Delta(W)$ is feasible on the vertices that are visited by $W$. This can be done simply in $O(|W|)$ by traversing the walk $[W, W]$ and checking if the time to expiry for any of the visited vertices becomes negative. Given this function, the simple greedy algorithm is to pick the vertex $i=\argmin\{s_j\} $ subject to the constraint that \textsc{PeriodicFeasibility$([W , i],G)$} returns true, where $W$ is the walk traversed so far. The algorithm terminates when all the vertices are either expired, or covered by the walk.

\subsection{Orienteering Based Greedy Algorithm}
Algorithm~\ref{alg:greedy} also finds partitioned walks by finding a feasible walk on a subset of vertices and then considering the remaining subgraph. The idea is to visit more vertices on the way to the greedily picked vertex. From the current vertex $x$, the target vertex $y$ is picked greedily as described in Section~\ref{sec:simple_greedy}. Then the time $d$ is calculated in line~\ref{algln:dist} which is the maximum time to go from $x$ to $y$ for which the periodic walk remains feasible. In line~\ref{algln:app_walk}, \textsc{Orienteering}$(V-V_{\texttt{exp}},x,y,d,\psi)$ finds a path in the vertices $V-V_{\texttt{exp
}}$ from $x$ to $y$ of length at most $d$ maximizing the sum of the weights $\psi$ on the vertices of the path. The set $V_{\texttt{exp
}}$ represents the expired vertices whose latencies cannot be satisfied by the current walk, and they will be considered by the next robot. The vertices with less time to expiry are given more importance in the path by setting weight $\psi_i=1/s_i$ for vertex $i$. The vertices that are already in the walk will remain feasible, and so their weight is discounted by a small number $m$ to encourage the path to explore unvisited vertices.

\begin{algorithm}[h]
	\openup -.1em
	\caption{\textsc{OrienteeringGreedy}}
	\label{alg:greedy}
	\begin{algorithmic}[1]
	\Statex Input: Graph $G=(V,E)$, latency constraints $r(v), \forall v$
	\Statex Output: A set of walks $\W$, such that $L(\W,v)\leq r(v)$
  \vspace{0.2em}
  \hrule
  \vspace{0.2em}
	\State $j=1$, $\W = \{\}$
	\While {$V$ is not empty}
	\State $V_{\texttt{exp}}=\{\}$
	\State $s_i = r(i)$ for all $i\in V$
	\State $W_j = \text{pick vertex } a \text{ randomly from } V $
	\While {$V-V(W_j)-V_{\texttt{exp}}$ is not empty}\label{algln:while}
	\State $x = $ last vertex in $W_j$
	\For {$y\in V-V_{\texttt{exp}}$ in increasing order of $s$}
	\If {\textsc{PeriodicFeasibility$([W_j , y],G)$}}
	\State Use binary search between $l(x,y)$ and $s_y$ \hspace*{21mm}to get $d$ (time to go from $x$ to $y$) such that \hspace*{21mm}$[W_j,y]$ remains feasible\label{algln:dist}
	\For {$z$ in $V-(V_{\texttt{exp}}\cup V(W_j))$}
	\If {$s_z < d + l(y,a)$} $V_{\texttt{exp}}=\hspace*{25mm}V_{\texttt{exp}}\cup z$\EndIf\label{algln:newif}
	\EndFor
	\State $\psi_i=1/s_i$ for all non expired vertices $i$
	\State $\psi_i = m \psi_i$ for $i$ in $V(W_j)$
	\State $W_j = [W_j, $\textsc{Orienteering}$(V-\hspace*{21mm}V_{\texttt{exp}},x,y,d,\psi)]$\label{algln:app_walk}
	\State Update $s$ using Equation~(\ref{eq:slackupadate})
	%\State Go to line~\ref{algln:while}
	\Else
	\State $V_{\texttt{exp}}=V_{\texttt{exp}}\cup y$
	\EndIf
	\EndFor
	\EndWhile
	\State $\W = \{\W,W_j\}$, $j = j+1$\label{algln:sol_set}
	\State $V = V-V(W_j)$\label{algln:reduce_size}
	\EndWhile
	\end{algorithmic}
\end{algorithm}

\begin{lemma}
Algorithm~\ref{alg:greedy} returns a feasible solution, i.e., for the set of walks $\W$ returned by Algorithm~\ref{alg:greedy}, $L(\W,v)\leq r(v)$, for all $v\in V$.
\end{lemma}
\begin{proof}
The vertices covered by the walk $W_j$ added to the solution in line~\ref{algln:sol_set} are removed from the set of vertices before finding the rest of the walks. Hence the latencies of the vertices $V(W_j)$ will be satisfied by only $W_j$. We will show that every time $W_j$ is appended in line~\ref{algln:app_walk}, it remains feasible on $V(W_j)$.

$W_j$ starts from a single vertex $a$, and hence is feasible at the start. Let us denote $W_j^-$ as the walk before line~\ref{algln:app_walk} and $W_j^+$ as the walk after line~\ref{algln:app_walk}. Due to line~\ref{algln:dist}, if $W_j^-$ is feasible in a particular iteration, then $[W_j^-,y]$ will remain feasible. Hence the only vertices than can possibly have their latency constraints violated in $W_j^+$ are in the orienteering path from $x$ to $y$. Consider any vertex $z$ in the path from $x$ to $y$ returned by the \textsc{Orienteering} function in line~\ref{algln:app_walk}. If $z\in V(W_j^-)$, then $L(W_j^+,z)\leq r(z)$ because of line~\ref{algln:dist}. If $z\notin V(W_j^-)$, then $r(z)=l(W_j^-) + s_z$ and by line~\ref{algln:newif}, $r(z)\geq l(W_j^-) + d + l(y,a)$. As $z$ is only visited once in $W_j^+$, $ L(W_j^+,z) = l(W_j^+) \leq l(W_j^-)+d+l(y,a)  \leq r(z)$.
\end{proof}

An approximation algorithm for \textsc{Orienteering} can be used in line~\ref{algln:app_walk} of Algorithm~\ref{alg:greedy}. In our implementation, we used an ILP formulation to solve \textsc{Orienteering}. To improve the runtime in practice, we pre-process the graph before calling the \textsc{Orienteering} solver to consider only the vertices $z$ such that $l(x,z)+l(z,y)\leq d$. We show in the next section that although the runtime of Algorithm~\ref{alg:greedy} is more than that of Algorithm~\ref{approx_latency}, it can still solve instances with up to $90$ vertices in a reasonable amount of time, and it finds better solutions.
\section{Simulation Results}
\label{sec:sim}
We now present the performance of the algorithms presented in the paper. For the approximation algorithm, we used the LKH implementation~\cite{helsgaun2000effective} to find the TSP of the graphs instead of the Christofides approximation algorithm~\cite{christofides1976worst}. This results in the loss of approximation guarantee but gives better results in practice. The orienteering paths in Algorithm~\ref{alg:greedy} were found using the ILP formulation from~\cite{letchford2013compact} and the ILP's were solved using the Gurobi solver~\cite{gurobi}.
\subsection{Patrolling an Environment}
The graphs for the problem instances were generated randomly in a real world environment. The scenario represents a ground robot monitoring the University of Waterloo campus. Vertices around the campus buildings represent the locations to be monitored and a complete weighted graph was created by generating a probabilistic road-map to find paths between those vertices. Figure~\ref{fig:environment} shows the patrolling environment. To generate random problem instances of different sizes, $n$ random vertices were chosen from the original graph. The latency constraints were generated uniformly randomly between $\text{TSP}/k$ and $k\text{TSP}$ where $k$ was chosen randomly between $4$ and $8$ for each instance. Here $\text{TSP}$ represents the TSP length of the graph found using LKH.

For each graph size, $10$ random instances were created. The average run times of the algorithms are presented in Figure~\ref{fig:runtimes}. As expected, Algorithm~\ref{alg:greedy} is considerably slower than the approximation and simple greedy algorithms due to multiple calls to the ILP solver. However, as shown in Figure~\ref{fig:obj_value}, Algorithm~\ref{alg:greedy} also gives the minimum number of robots for most of these instances. %A similar trend was observed when latency constraints were exponentially distributed between $\text{TSP}/4$ and $\infty$ with mean $3\text{TSP}$.
\begin{figure}[t]
\centering
\includegraphics[width=.5\linewidth, angle=90]{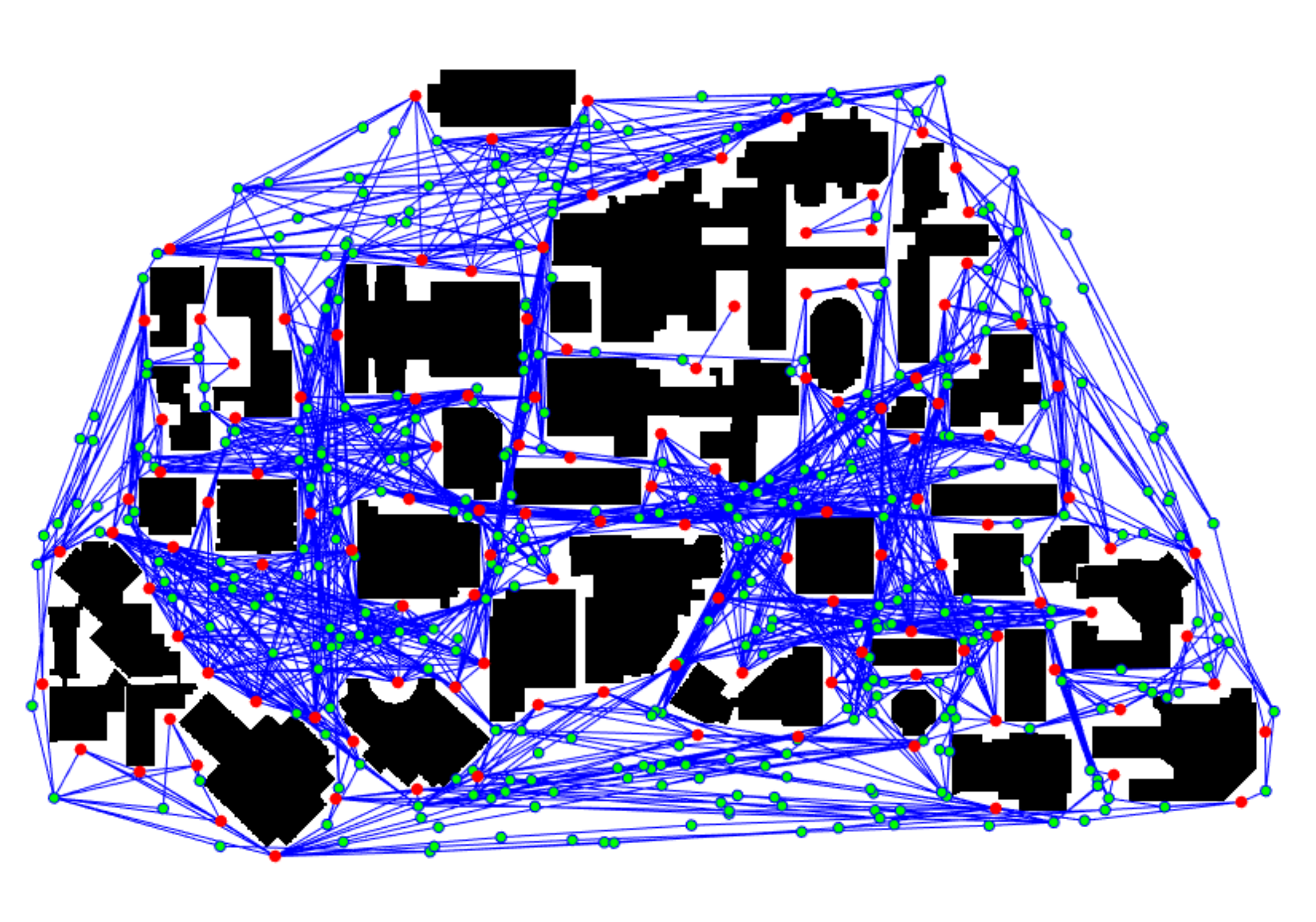}
\caption{The environment used for generation of random instances. The red dots represent the vertices to be monitored and the green dots represent the vertices in the PRM used to find shortest paths between red vertices.}
\label{fig:environment}
\end{figure}
\begin{figure}[t]
\centering
\begin{subfigure}{.9\textwidth}
\centering
\includegraphics[width=\linewidth]{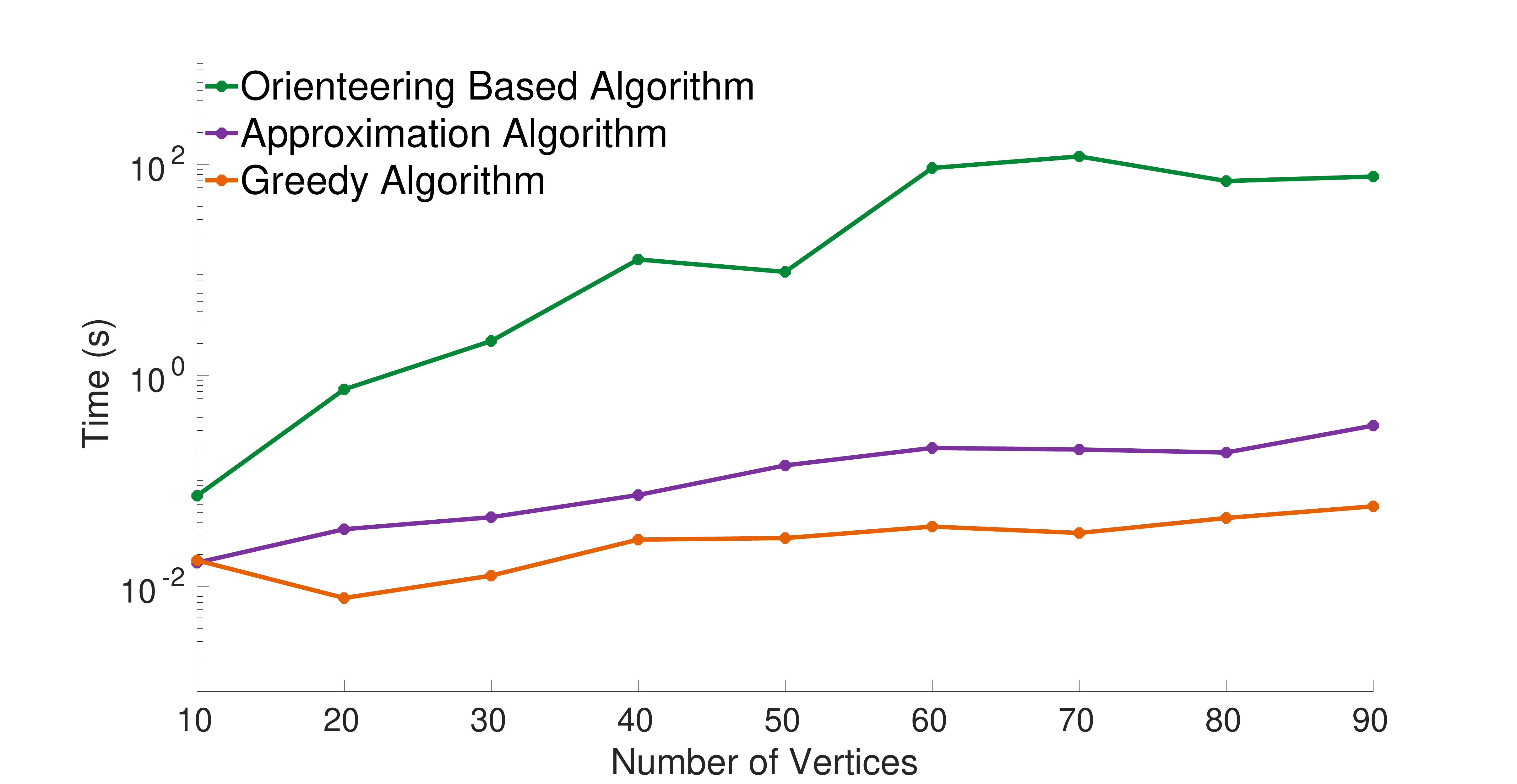}
\captionof{figure}{}
\label{fig:runtimes}
\end{subfigure}\\
\begin{subfigure}{.9\textwidth}
\centering
\includegraphics[width=\linewidth]{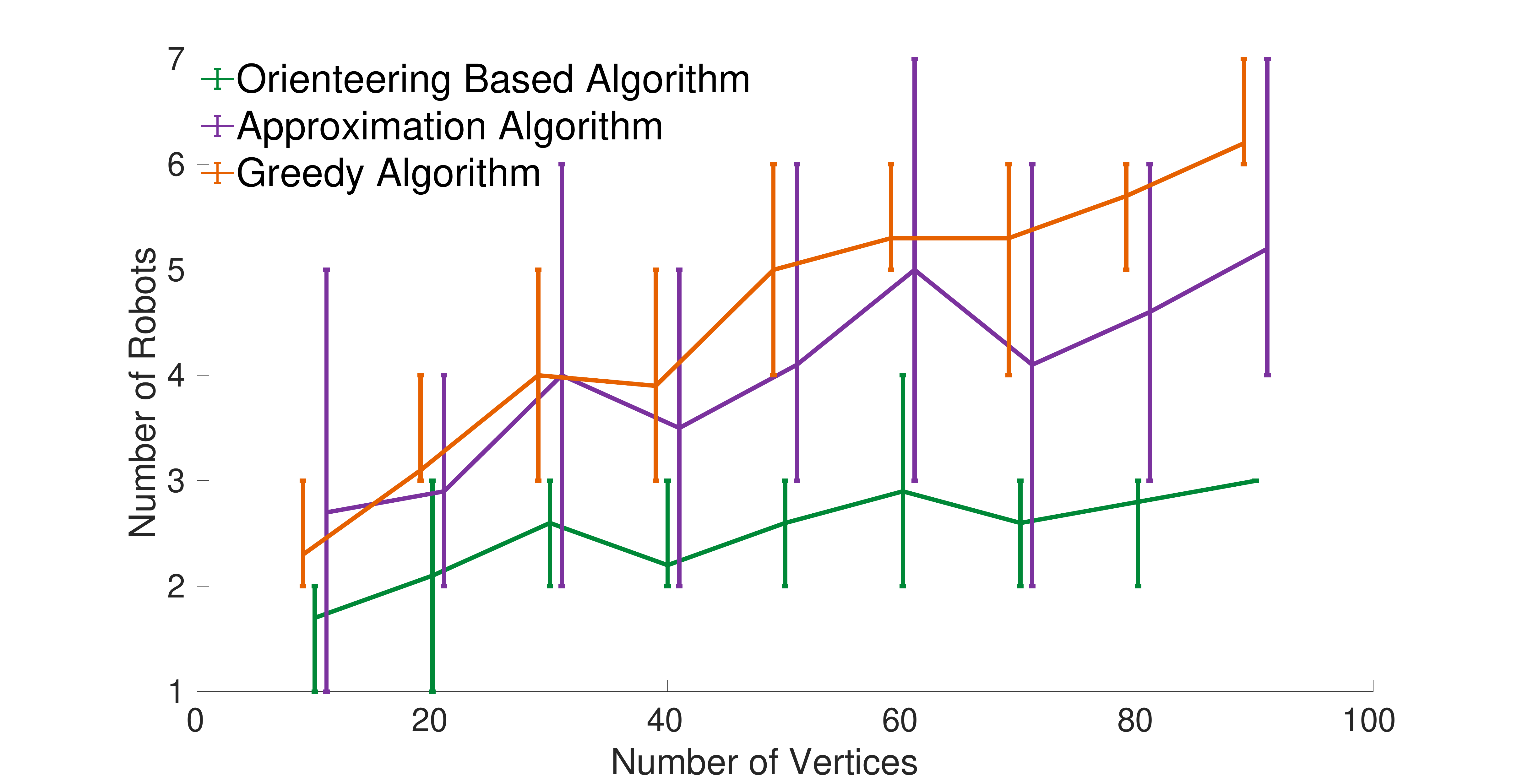}
\captionof{figure}{}
\label{fig:obj_value}
\end{subfigure}
\caption{Average run times of the algorithms (a), and number of robots returned by each algorithm (b). The line plot shows the mean over $10$ random instances for each graph size. The error bars in (b) show the minimum and maximum number of robots required for a graph size. }
\label{fig:alg_comp}
\end{figure}

\subsection{Persistent 3D Scene Reconstruction}
Another application of our algorithms is in capturing images for 3D reconstruction of a scene.  Since existing algorithms focus on computing robot paths to map a static scene~\cite{bircher2015structural,roberts2017submodular}, our algorithms could be applied to persistently monitor and thus maintain an up-to-date reconstruction of a scene that changes over time. To demonstrate this, we create problem instances using a method similar to~\cite{roberts2017submodular}. The viewpoints were generated on a grid around the building to be monitored. For each viewpoint, five camera angles were randomly generated, and best angle was selected for each viewpoint based on a view score that was calculated assuming a square footprint for the camera. For each camera angle, equally placed rays were projected onto the building within the footprint and a score was calculated based on the distance and incidence angle of the ray. This calculation is similar to that in~\cite{roberts2017submodular}, although they used a more detailed hemisphere coverage model.

After selecting the viewing angles, the final score of a camera pose was evaluated as in~\cite{roberts2017submodular} by greedily picking the best viewpoint first and evaluating the marginal score of other viewpoints. The resulting graph had $109$ vertices. The latencies were set such that the most informative viewpoint is visited every $8$ minutes and on average each viewpoint is visited every $50$ minutes. Algorithm 3 found a solution in $150$ seconds using two walks, as shown in Figure~\ref{fig:taj}.
Note that Algorithm~\ref{approx_latency} returned a solution with $3$ robots.
\begin{figure}[t]
\centering
\begin{subfigure}{.5\textwidth}
\centering
\includegraphics[width=\linewidth]{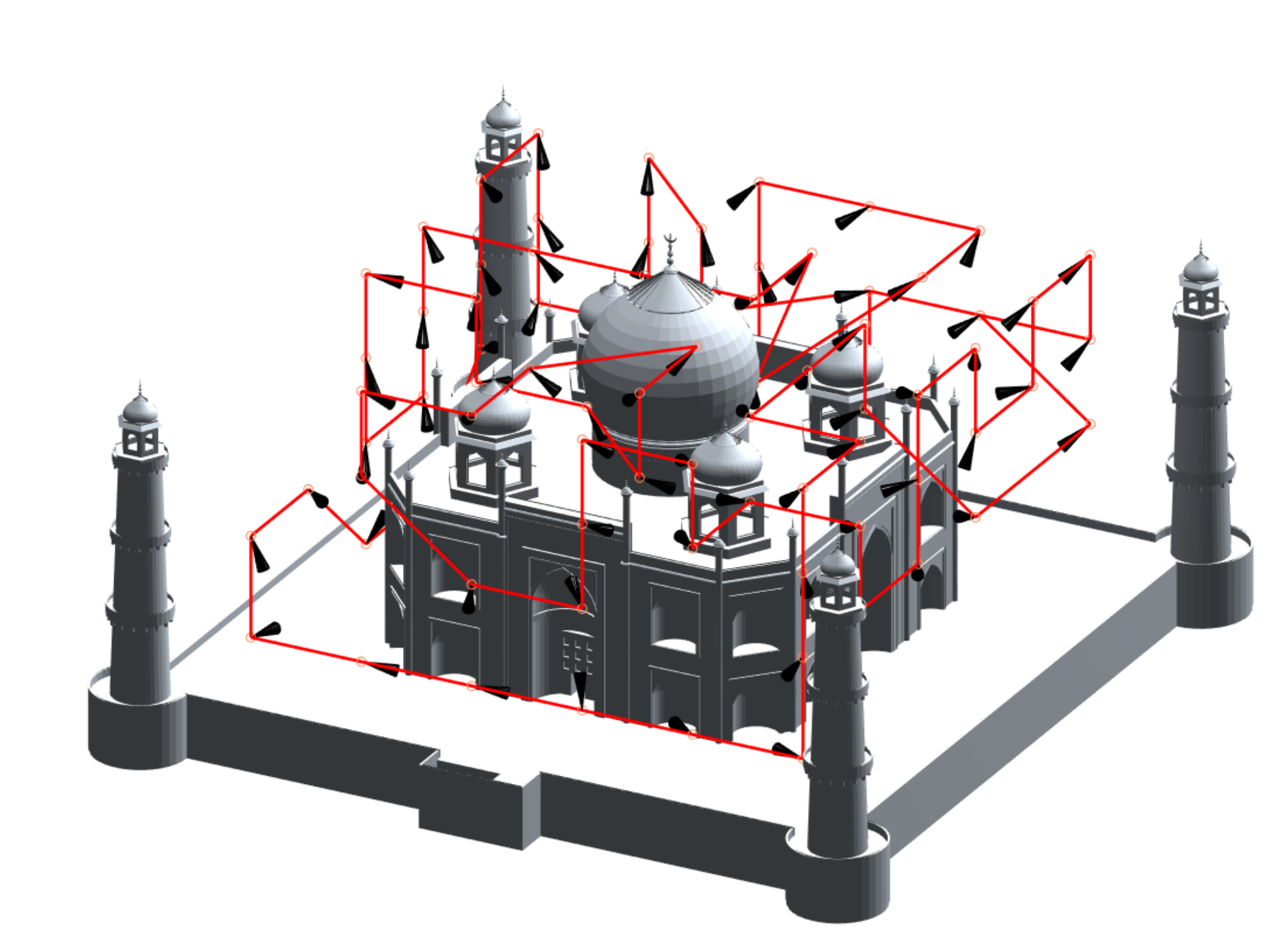}
\label{fig:taj1}
\end{subfigure}%
\begin{subfigure}{.5\textwidth}
\centering
\includegraphics[width=\linewidth]{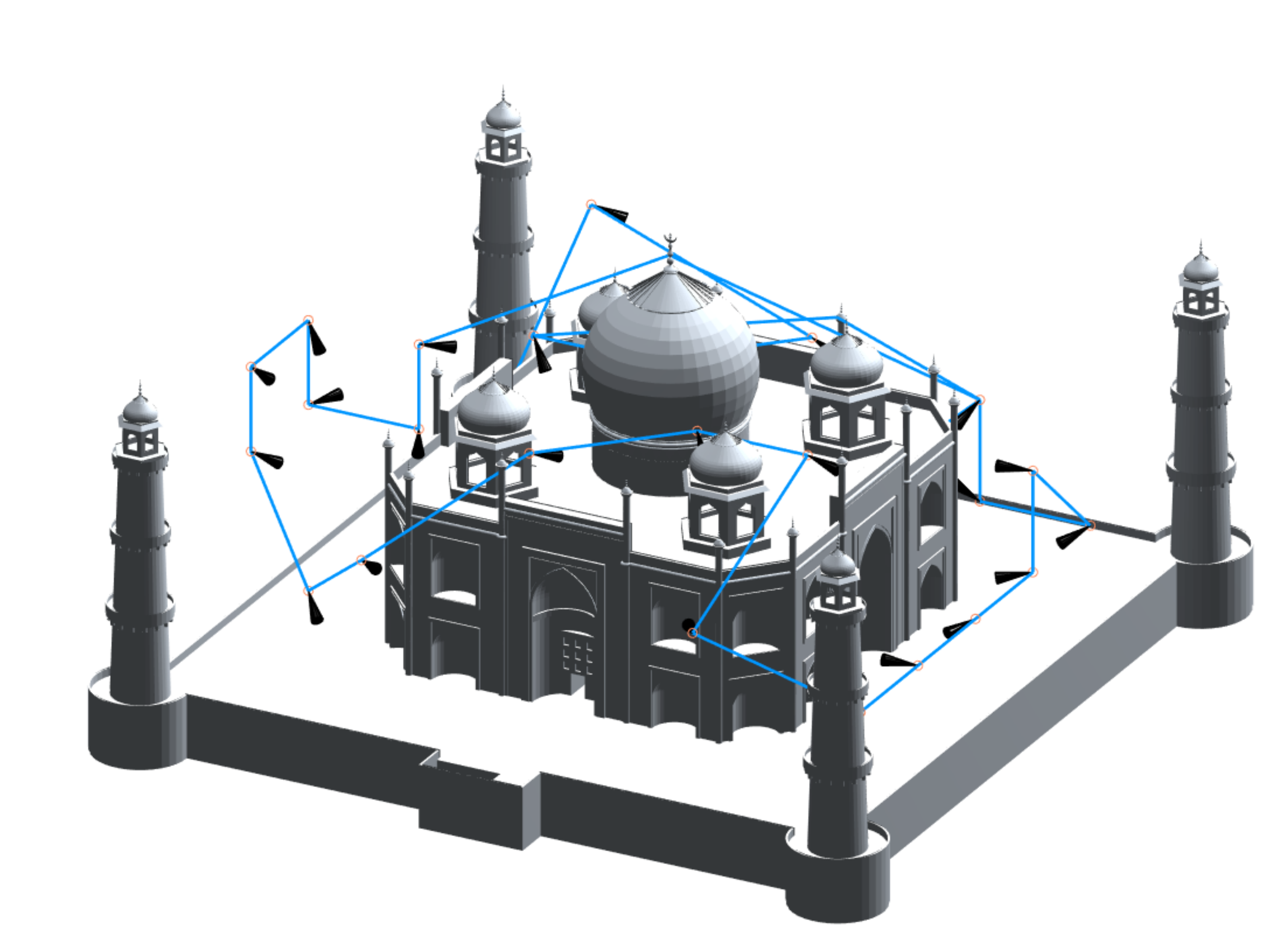}
\label{fig:taj2}
\end{subfigure}
\caption{The walks returned by Algorithm~\ref{alg:greedy} to continually monitor the mausoleum of the Taj Mahal. The cones at each viewpoint show the camera angle. Note that the walk on the left is not a tour, as it visits the vertex with least latency twice within a period. The walk on the right is a tour and it visits the vertices that the first robot was unable to cover.}
\label{fig:taj}
\end{figure}
\subsection{Comparison with Existing Algorithms in Literature}
In~\cite{Drucker2014thesis,drucker2016cyclic} the authors propose an SMT (Satisfiability Modulo Theory) based approach using Z3 solver~\cite{de2008z3} to solve the decision version of the problem. The idea is to fix an upper bound on the period of the solution and model the problem as a constraint program. The authors also provide benchmark instances for the decision version of the problem. We tested our algorithms on the benchmark instances provided and compare the results to the SMT based solver provided by~\cite{drucker2016cyclic}.

Out of 300 benchmark instances, given a time limit of 10 minutes, the Z3 solver returned 182 instances as satisfiable with the given number of robots. We ran our algorithms for each instance and checked if the number of robots returned are less than or equal to the number of robots in the benchmark instance. The approximation algorithm satisfied $170$ instances whereas Algorithm~\ref{alg:greedy} satisfied $178$ instances. The four satisfiable instances that Algorithm~\ref{alg:greedy} was unable to satisfy had optimal solutions where the walks share the vertices, and Algorithm~\ref{alg:greedy} returned one more robot than the optimal in all those instances. The drawback of the constraint program to solve the problem is the scalability. It spent an average of $3.76$ seconds on satisfiable instances whereas Algorithm~\ref{alg:greedy} spent $3$ ms on those instances on average. Moreover, on one such instance where Algorithm~\ref{alg:greedy} returned one more robot than the Z3 solver, Z3 solver spent 194 seconds as compared to $\sim 5$ ms for Algorithm~\ref{alg:greedy}. Note that these differences are for benchmark instances having up to $7$ vertices. As shown above, Algorithm~\ref{alg:greedy} takes $\sim 100$ seconds for $90$ vertex instances whereas we were unable to solve instances with even $15$ vertices within an hour using the Z3 solver. Hence, the scalability of the Z3 based solver hinders its use for problem instances of practical sizes.

\section{Conclusion and Future Work}
In this paper we presented and analyzed an approximation and a heuristic algorithm for the problem of finding the minimum number of robots that can satisfy the latency constraints for the vertices in a graph. We demonstrated the performance of the algorithms through simulations. We also presented and analyzed an algorithm for minimizing the maximum latency given multiple robots. Finding the relation between the partitioned optimal solution and a general optimal solution is an interesting direction for future work.
%
% ---- Bibliography ----
%

%

\end{document}